\pdfoutput = 1

\documentclass[11pt]{article}
\usepackage{fullpage}

\usepackage[utf8]{inputenc} % allow utf-8 input
\usepackage[T1]{fontenc}    % use 8-bit T1 fonts
\usepackage{hyperref}       % hyperlinks
\usepackage{url}            % simple URL typesetting
\usepackage{booktabs}       % professional-quality tables
\usepackage{amsfonts}       % blackboard math symbols
\usepackage{nicefrac}       % compact symbols for 1/2, etc.
\usepackage{microtype}      % microtypography
\usepackage{xcolor}         % colors

\usepackage[pdftex]{graphicx}
\usepackage{subcaption}

\usepackage{amsfonts}       % blackboard math symbols 
\usepackage{amssymb}
\usepackage{amsmath}
\usepackage{amsthm}
\usepackage{mathtools}
\usepackage{multicol}
\usepackage{multirow}

\usepackage{color}

\usepackage{array}
\usepackage{algorithm}
\usepackage{algorithmic}

\usepackage[shortlabels]{enumitem}

\usepackage[normalem]{ulem}
\newcounter{ToDo}
\newcounter{gaocomm} 
\newcounter{Note}
\definecolor{blue-violet}{rgb}{0.00,0.75,0.90}
\definecolor{mygreen}{rgb}{0.0, 0.5, 0.0}
\definecolor{awesome}{rgb}{1.0, 0.13, 0.32}
\definecolor{bostonuniversityred}{rgb}{1.0, 0.0, 0.0}
\usepackage[symbol]{footmisc}

\def\bvarphi{{\boldsymbol \varphi}}
\def\bpsi{{\boldsymbol \psi}}
\def\btheta{{\boldsymbol \theta}}
\def\bphi{{\boldsymbol \phi}}
\def\bones{{\boldsymbol 1}}

\def\be{{\mathbf e}}

\def\bu{\mathbf u}
\def\bv{\mathbf v}
\def\bw{{\mathbf w}}

\def\bz{{\mathbf z}}

\def\bA{\mathbf A}

\def\bD{\mathbf D}

\def\bF{\mathbf F}

\def\bH{\mathbf H}
\def\bI{{\mathbf I}}

\def\bL{\mathbf L}

\def\bP{{\mathbf P}}

\def\bS{\mathbf S}

\def\bU{{\mathbf U}}
\def\bV{{\mathbf V}}
\def\bW{{\mathbf W}}
\def\bX{{\mathbf X}}

\def\bf{{\mathbf f}}
\def\bh{{\mathbf h}}

\def\sR{{\mathbb R}}

\def\gE{{\mathcal{E}}}

\def\gW{{\mathcal{W}}}

\def\gV{{\mathcal{V}}}
\def\gG{{\mathcal{G}}}

\def\gI{{\mathcal{I}}}

\def\vec{{\mathrm{vec}}}

\newcommand{\trace}{\mathrm{tr}}

\newtheorem{theorem}{Theorem}
\newtheorem{lemma}{Lemma}
\newtheorem{proposition}{Proposition}

\theoremstyle{definition}
\newtheorem{definition}{Definition}

\newtheorem{remark}{Remark}

\def\bLambda{{\mathbf \Lambda}}

\def\bTheta{{\mathbf \Theta}}

\def\bDelta{{\mathbf \Delta}}

\def\bOmega{{\boldsymbol \Omega}}

\title{\textbf{Generalized energy and gradient flow \\ via graph framelets}}

\author{Andi Han\footnote{University of Sydney 
   (\texttt{andi.han@sydney.edu.au}, \texttt{zsha2911@uni.sydney.edu.au}, \texttt{junbin.gao@sydney.edu.au})}
\and Dai Shi\footnote{Western Sydney University \texttt{(dai.shi@sydney.edu.au})}
\and Zhiqi Shao\footnotemark[1]
\and Junbin Gao\footnotemark[1]
}
\date{}

\begin{document}

\maketitle

\begin{abstract}
    In this work, we provide a theoretical understanding of the framelet-based graph neural networks through the perspective of energy gradient flow. By viewing the framelet-based models as discretized gradient flows of some energy, we show it can induce both low-frequency and high-frequency-dominated dynamics, via the separate weight matrices for different frequency components. This substantiates its good empirical performance on both homophilic and heterophilic graphs. We then propose a generalized energy via framelet decomposition and show its gradient flow leads to a novel graph neural network, which includes many existing models as special cases. We then explain how the proposed model generally leads to more flexible dynamics, thus potentially enhancing the representation power of graph neural networks.
\end{abstract}

\section{Introduction}
Graph neural networks (GNNs) \cite{gasteiger2018predict,hamilton2017inductive,kipf2016semi,velivckovic2018graph,wu2019simplifying} have become the primary tool for representation learning over graph-structured data, such as social networks \cite{chen2018fastgcn} citation networks \cite{kipf2016semi}, molecules \cite{duvenaud2015convolutional}, traffic networks \cite{cui2019traffic}, among others. There generally exists two types of GNN models, i.e., spatial and spectral based models. Spatial GNNs, including MPNN \cite{gilmer2017neural}, GAT \cite{velivckovic2018graph}, GIN \cite{xu2018powerful}, usually propagate information in the neighbourhood and update node representations via a weighted average of the neighbours. Spectral GNNs, including ChebyNet \cite{defferrard2016convolutional}, GCN \cite{kipf2016semi}, BernNet \cite{he2021bernnet}, performs filtering on the spectral domain provided by graph Fourier transform (where the orthonormal system is given by the eigenvectors of graph Laplacian). Wavelet-based graph representation learning \cite{dong2017sparse,hammond2011wavelets,li2020fast,wang2021deep,xu2018graph,zheng2021framelets,zheng2020mathnet,zheng2022decimated,zou2022simple}, a class of spectral methods, provide a multi-resolution analysis of graph signals and thus often lead to better signal representations by capturing information at different scales. In particular, graph framelets \cite{dong2017sparse,zheng2022decimated}, a type of tight wavelet frame, further allows separate modelling of low-pass and high-pass signal information, and has been considered to define graph convolution as in \cite{chendirichlet,zheng2021framelets,zheng2022decimated}. It has been shown  that the graph framelet convolution allows more flexible control of approximate (via low-pass filters) and detailed information (via high-pass filters) with great robustness and efficiency, achieving the state-of-the-art results on multiple graph learning tasks \cite{chendirichlet,yang2022quasi,zheng2021framelets,zhou2021graph,zhou2021spectral,zou2022simple}.

Along with the developments of more advanced models, theoretical understanding of both the power and pitfalls of graph neural networks has called for great attention. Currently, there exists many known limitations of GNNs, including over-smoothing \cite{cai2020note,li2019deepgcns}, over-squashing \cite{alon2020bottleneck,topping2021understanding}, limited expressive power \cite{maron2019provably,xu2018powerful} and poor performance on heterophilic graphs \cite{pei2019geom,zhu2020beyond}. In an attempt to better resolve the above issues, many studies have tried to understand GNNs through various frameworks, such as dynamical systems and differential equations \cite{chamberlain2021grand,di2022graph,oono2019graph,thorpe2021grand,wang2021dissecting}, Riemannian geometry \cite{ni2019community,topping2021understanding}, algebraic topology \cite{bodnar2022neural,hansen2020sheaf}. Nevertheless, existing analyses are limited to either spatial GNNs or spectral GNNs based on the Fourier transform. In fact, it has been empirically observed that wavelet/framelet-based models tend to alleviate the previously known issues, such as mitigating the issue of over-smoothing \cite{wang2021deep} and achieving good performance on heterophilic graphs \cite{chendirichlet}. Despite the success of wavelet-based models, comparatively little is known on how the multi-scale and frequency-separation properties provided by graph wavelets/framelets potentially avoid the aforementioned pitfalls and enhance the learning capacity of GNNs. 
% \EScomment{Can we try to swap the last two sencetences, it would be better to say currently we only have empirical observation rather than theoritical support?}
% \AHcomment{Done}

In this paper, we particularly focus on the graph framelet-based models and analyze the model behaviors from the perspective of energy gradient flows. In physics, the evolution 
% (\EScomment{dynamic?}) \AHcomment{Just another word} 
of particles is often modelled as differential equations that minimize an energy, known as the gradient flows. The energy functional and its gradient flow provide a characterization of particles' states and movements, which are essential to understand the dynamics. Recently, such idea has been adapted to graph neural networks \cite{di2022graph}. By modelling GNNs as (discretized) gradient flows, one can study the limiting behaviors of GNNs, which are closely related to the notions of over-smoothing, over-separating and heterophilic graphs. Particularly, \cite{di2022graph} characterizes the dynamics of GNNs in terms of low-frequency or high-frequency dominance or both. It is known that low-frequency-dominant (LFD) models tend to perform well on homophilic graphs (where neighbouring nodes are likely to come from the same cluster) \cite{nt2019revisiting,wu2019simplifying,klicpera2019diffusion,di2022graph} while in contrast, high-frequency-dominant (HFD) models tend to perform well on heterophilic graphs \cite{bo2021beyond,di2022graph}.
% \EScomment{Citation?}. 
In this regard, models that can be both LFD and HFD are usually preferred. 

To the best of our knowledge, this is the first work that provides a theoretical understanding of multi-scale graph neural networks and justification for its good empirical performance. Specifically, we highlight our main contributions as follows. 
\begin{itemize}
    \item We show the framelet convolutions, either spatial-based \cite{chendirichlet} or spectral-based \cite{zheng2021framelets} can be viewed as a discretized gradient flow of some energy. From this point of view, we prove that the framelet-based GNNs can induce both LFD and HFD dynamics, thus theoretically explaining its effectiveness on heterophilic graphs and its capability to potentially avoid over-smoothing, as often empirically observed. 
    
    \item We then define a generalized energy via framelet decomposition. We show such energy includes the energy proposed in \cite{di2022graph}, which itself is a generalization of the graph Dirichlet energy. We then propose a GNN model, namely gradient flow based framelet graph convolution (GradF-UFG), as discretization of the proposed generalized energy. 
    
    \item We show the proposed GradF-UFG includes the framelet-based GNNs as special cases. We also connect the proposed model with the recently introduced energy enhanced framelet convolution and analyze its behaviors. We explain how the proposed model provides more flexibility compared to existing works. 
\end{itemize}

\paragraph{Organization.}
The rest of the paper is organized as follows. Section \ref{sect_related_work} first summarizes 
% \EScomment{A bit wordy here, consider use summarizes} 
the related works on continuous
% \EScomment{Maybe link equation with inspired as well} 
GNNs and its connections to dynamical systems and differential equations. We also provide an overview on framelet-based graph representation learning. In section \ref{sect_prelinm}, we review the preliminary knowledge on graphs, graph framelets and framelet convolutions. We also introduce the notions of graph Dirichlet energy, gradient flow along with the definitions of LFD and HFD dynamics. Section \ref{sect_dirichlet_energy_framelet} starts with identifying a connection between the spatial framelet convolution with Dirichlet energy and also characterizes its asymptotic behaviors from the lens of gradient flow. In section \ref{sect_genalized_energy}, we define the framelet generalized energy and propose a GNN model based on its gradient flow, where we also connect to existing works. Finally, in section \ref{sect_spec_grad_flow}, we perform similar analysis for the spectral framelet convolution.

\section{Related works}
\label{sect_related_work}

\paragraph{Dynamical systems, differential equations and GNNs.} 
Neural ODE \cite{chen2018neural} is introduced as a continuous version of ResNet \cite{he2016deep}. Since then, many work has studied its counterparts on graphs and proposed continuous GNNs, such as \cite{avelar2019discrete,poli2019graph,xhonneux2020continuous}. Another parallel research direction aims to understand GNNs from continuous dynamical systems and differential equations. GCN \cite{kipf2016semi}, one of the most popular GNN models, can be viewed as a discrete Markov process \cite{oono2019graph} and its linear version is verified as the discretized (isotropic) graph heat diffusion \cite{wang2021dissecting}, which minimizes the graph Dirichlet energy. Similarly, GAT \cite{velivckovic2018graph} is related to an anisotropic heat diffusion on graphs as shown in \cite{chamberlain2021grand}. Furthermore, many recent studies introduce GNNs that are inspired by physical systems and diffusion PDEs, including heat diffusion with a source term \cite{thorpe2021grand}, non-Euclidean Beltrami flow \cite{chamberlain2021beltrami}, wave diffusion equation \cite{eliasof2021pde}, networks of coupled oscillators \cite{rusch2022graph}, nonlinear anisotropic diffusion \cite{chen2022optimization}. The recent work \cite{di2022graph} provides a framework for analyzing continuous formulation of GNN models as gradient flows of some energy. In particular, the work proposes a general energy where its gradient flow leads to many existing GNN models. Under such framework, \cite{di2022graph} verifies that many models can only lead to LFD dynamics, including GCN \cite{kipf2016semi}, GRAND \cite{chamberlain2021grand}, CGNN \cite{xhonneux2020continuous} and PDE-GCN \cite{eliasof2021pde}.

\paragraph{Framelet-based graph learning.}
The work \cite{dong2017sparse} provides an efficient way to compute graph framelet transforms via Chebyshev polynomial approximation, which allows practical applications such as semi-supervised clustering. Graph spectral framelet convolution has been proposed in \cite{zheng2021framelets,zheng2022decimated}, which is shown to empirically enhance the graph neural networks with its robustness and multi-scale properties. Later, graph framelets have been integrated for graph signal denoising \cite{zhou2021graph}, robust graph embedding \cite{zhou2022graph}, dynamic graph \cite{zhou2021spectral} and directed graph learning \cite{zou2022simple}, exhibiting great improvement in model performance. More recently, \cite{chendirichlet} proposes a spatial graph framelet convolution and shows its close connection to the Dirichlet energy. The paper also introduces a perturbed Dirichlet energy in order to mitigate the over-smoothing issue.

\section{Preliminaries}
\label{sect_prelinm}

\paragraph{Graphs and graph convolution.}
A graph $\gG = (\gV_\gG, \gE_\gG)$ of $n$ nodes (with self-loops) can be represented by graph adjacency matrix $\bA \in \sR^{n \times n}$. In this work, we assume the graph is undirected and unweighted, i.e., $\bA$ is symmetric with $a_{ij} =1$ if $(i,j) \in \gE_\gG$ and $0$ otherwise. In addition, we consider the symmetric normalized adjacency matrix as $\widehat{\bA} = \bD^{-1/2} \bA \bD^{-1/2}$ where $\bD$ is the diagonal degree matrix, with the $i$-th diagonal entry given by $d_{i} = \sum_{j} a_{ij}$, the degree of node $i$. The normalized graph Laplacian is given by $\widehat{\bL} = \bI_n - \widehat{\bA}$. We use $\rho_{L}$ to denote the largest eigenvalue (also called the highest frequency) of $\widehat{\bL}$. From the spectral graph theory \cite{chung1997spectral}, $\rho_L \leq 2$ and the equality holds if and only if there exists a connected component of the graph $\gG$ that is bipartite. 

Graph convolution network (GCN) \cite{kipf2016semi} defines the layer-wise propagation rule via the normalized adjacency matrix as
\begin{equation}
    \bH(\ell + 1) = \sigma \big( \widehat{\bA} \bH(\ell) \bW^\ell  \big), \label{eq_classic_gcn}
\end{equation}
where $\bH(\ell)$ denotes the feature matrix at layer $\ell$ with $\bH(0) = \bX \in \sR^{n \times c}$, the input features (also called input signals), and $\bW^\ell$ is the learnable feature transformation. It can be shown that GCN corresponds to a localized filter via graph Fourier transform, i.e., $\bh(\ell+1) = \bU^\top (\bI_n - \bLambda) \bU \bh(\ell)$ 
% \GaoC{I think using $\bh(\ell+1) = \bU (\bI_n - \bLambda) \bU^\top \bh(\ell)$ is better as the columns of $\bU$ are eigen signals as convention}. \AHcomment{Can we just use this one? Because I follow \cite{chendirichlet} where they use $U$. Also I need to change all $U$ to $U^\top$. I think also because we have $\gW^\top$ in the front. So I think we can stick to this one just to align with $\gW^\top$. :)} 
where $\bU, \bLambda$ are from the eigendecomposition $\widehat{\bL} = \bU^\top \bLambda \bU$ and $\bU \bh$ is known as the Fourier transform of a graph signal $\bh \in \sR^n$. In this paper, we let $\{ (\lambda_i, \bu_i) \}_{i=1}^n$ be the set of eigen-pairs of $\widehat{\bL}$ where $\bu_i$ are the row vectors of $\bU$.
% \GaoC{Can we specify that $\mathbf u_i$ are in the rows of $\mathbf U$?}
% \AHcomment{Done}

\paragraph{Graph framelets and framelet convolution.}
Graph (undecimated) framelets \cite{zheng2022decimated} are defined via a filter bank $\eta = \{ a; b^{(1)}, ..., b^{(L)} \}$ and its induced (complex-valued) scaling functions $\Psi = \{ {\alpha}; \beta^{(1)}, ..., \beta^{(L)} \}$ where $L$ is the number of high-pass filters. Particularly, it satisfies that $\widehat{\alpha}(2\xi) = \widehat{a}(\xi) \widehat{\alpha}(\xi)$ and $\widehat{\beta^{(r)}}(2\xi) = \widehat{b^{(r)}}(\xi) \widehat{\alpha}(\xi)$ for all $\xi \in \sR, r =1,...,L$ where $\widehat{\alpha}, \widehat{\beta^{(r)}}$ denote the Fourier transform of $\alpha, \beta^{(r)}$ and $\widehat{a}, \widehat{b^{(r)}}$ denote the Fourier series of $a, b^{(r)}$ respectively. The graph framelets are defined by $\varphi_{j,p}(v) = \sum_{i=1}^n \widehat{\alpha}\big( {\lambda_i}/{2^j} \big) u_i(p) u_i(v)$ and $\psi^r_{j,p}(v) = \sum_{i = 1}^n \widehat{\beta^{(r)}} \big( \lambda_i/ 2^j \big) u_i(p) u_i(v)$ for $r = 1,..., L$ and for scale level $j = 1,...,J$. We use $u_i(v)$ to represent the eigenvector $\bu_i$ at node $v$. $\varphi_{j,p}$ and $\psi^r_{j,p}$ are known as the \textit{low-pass framelets} and \textit{high-pass framelets} at node $p$. 

The \textit{framelet coefficients} of a graph signal $\bh$ are given by $\bv_{0} = \{ \langle \bvarphi_{0,p} , \bh \rangle \}_{p \in \gV_\gG}$,
% \GaoC{$V$ is inconsistent with $\mathcal{V}_{\mathcal{G}}$ from the beginning},
% \AHcomment{Done}
$\bw_{j}^r = \{ \langle \bpsi_{j,p}^r, \bh \rangle \}_{p \in \gV_\gG}$. For a multi-channel signal $\bH \in \sR^{n \times c}$, we can compactly write its framelet coefficients as
\begin{equation*}
    \bV_0 = \bU^\top \widehat{\alpha} \Big( \frac{\bLambda}{2} \Big) \bU \bH, \quad \bW_j^r = \bU^\top \widehat{\beta^{(r)}} \Big( \frac{\bLambda}{2^{j+1}} \Big) \bU \bH, \quad \forall j = 0,...,J, r = 1,...,L
\end{equation*}
where $\widehat{\alpha}, \widehat{\beta^{(r)}}$ applies elementwise to the diagonal of $\bLambda$, and $\bV_0, \bW_{j}^r$ are respectively the low-pass and high-pass coefficients. Define the framelet transform matrices $\gW_{0, J}, \gW_{r,j}$ such that $\bV_0 = \gW_{0,J} \bH = \bU^\top \bLambda_{0,J} \bU \bH$, and $\bW_{j}^r = \gW_{r,j} \bH = \bU^\top \bLambda_{r,j} \bU \bH$ for $r = 1,...,L, j = 1,...,J$, where $\bLambda_{r,j}$ is a diagonal matrix with entries $(\bLambda_{0,J})_{ii} = \widehat{\alpha}(\lambda_i/2)$ and $(\bLambda_{r,j})_{ii} = \widehat{\beta^{(r)}} (\lambda_i/2^{j+1})$. 
% \GaoC{Better to define the diagonal notation $\bLambda_{0,J}$, $\bLambda_{r,j}$} 
% \AHcomment{Done}
By the tightness of the framelet transform, we have $\bLambda_{0,J}^2 + \sum_{r,j} \bLambda_{r,j}^2 = \bI_n$ and the framelet decomposition and reconstruction are invertible, i.e., $\gW_{0, J}^\top \gW_{0,J} \bH + \sum_{r,j} \gW_{r,j}^\top \gW_{r,j} \bH = \bH$.\footnote{With a slight abuse of notations, we use $\top$ to also represent conjugate transpose if the filters are complex-valued.}  This property allows unique decomposition of any graph signal onto the spectral framelet domain. Refer to  \cite{zheng2021framelets,zheng2022decimated} for more detailed discussions. 

The \textit{spectral graph framelet convolution} is proposed in \cite{zheng2021framelets}, which is similar to the graph (Fourier) convolution by applying a filter on the spectral domain before reconstructing the signal. 
% Denote the index set $\gI = \{ (r,j) : r = 1,...,L, j = 1,...,J \} \cup \{ (0,J) \}$. 
The layer-wise propagation rule is given by $\bH(\ell + 1) = \sigma \big( \gW_{0,J}^\top {\rm diag}(\btheta_{0,J}) \gW_{0,J} \bH(\ell) \bW^\ell +  \sum_{r,j} \gW_{r,j}^\top {\rm diag}(\btheta_{r,j}) \gW_{r,j} \bH(\ell) \bW^\ell \big)$,
% \begin{equation}
%     \bH(\ell + 1) = \sigma \Big( \gW_{0,J}^\top {\rm diag}(\btheta_{0,J}) \gW_{0,J} \bH(\ell) \bW^\ell +  \sum_{r,j} \gW_{r,j}^\top {\rm diag}(\btheta_{r,j}) \gW_{r,j} \bH(\ell) \bW^\ell \Big), \label{eq_spec_frame}
% \end{equation}
where $\btheta_{r,j} \in \sR^{n}$ is a learnable filter coefficient and $\bW^\ell$ is a shared weight matrix across all $r,j$ for layer $\ell$. Rather than performing spectral filtering as in \cite{zheng2021framelets}, the \textit{spatial graph framelet convolution} performs a spatial message passing over the spectral framelet domain \cite{chendirichlet} as $\bH(\ell + 1) = \gW^\top_{0,J} \sigma \big( \widehat{\bA} \gW_{0,J} \bH(\ell)  \bW_{0,J}^{\ell} \big) + \sum_{r,j} \gW^\top_{r,j} \sigma \big( \widehat{\bA} \gW_{r,j} \bH(\ell)  \bW_{r,j}^{\ell} \big)$.

In this paper, our subsequent analysis focuses on the spatial framelet convolution (or just framelet convolution) due to its connection to the Dirichlet energy as shown in Section \ref{sect_dirichlet_energy_framelet}. Nevertheless, we also obtain similar conclusions for spectral framelet convolution in Section \ref{sect_spec_grad_flow}.

% \begin{equation}
%     \bH(\ell + 1) = \gW^\top_{0,J} \sigma \Big( \widehat{\bA} \gW_{0,J} \bH(\ell)  \bW_{0,J}^{\ell} \Big) + \sum_{r,j} \gW^\top_{r,j} \sigma \Big( \widehat{\bA} \gW_{r,j} \bH(\ell)  \bW_{r,j}^{\ell} \Big). \label{eq_spat_frame}
% \end{equation}

\paragraph{Dirichlet energy and gradient flow.}
The graph Dirichlet energy provides a measure of smoothness of a graph signal $\bH \in \sR^{n \times c}$, which is defined as $E(\bH) = \frac{1}{4} \sum_{i,j} a_{ij} \| \bh_i /\sqrt{d_i} - \bh_j /\sqrt{d_j} \|^2 = \frac{1}{2} \trace(\bH^\top \widehat{\bL} \bH)$, where $\bh_i \in \sR^c$ is the signal vector of node $i$. The gradient flow of the Dirichlet energy yields the so-called graph heat equation \cite{chung1997spectral} as $\dot{\bH}(t) = - \nabla E(\bH(t)) = - \widehat{\bL} \bH(t)$. Its Euler discretization leads to the propagation of linear GCN models \cite{wu2019simplifying,wang2021dissecting}. The process is called Laplacian smoothing \cite{li2018deeper} and it converges to 
% \GaoC{to a vector in?? or because ${\rm ker}(\widehat{\bL})$ is $\text{span}\{{\mathbf 1}\}$} \AHcomment{I think they are the same.} \GaoC{In theory ${\rm ker}(\widehat{\bL})$ is a subspace, not single vector.} \AHcomment{Yes, but what we say is the process converging to, so not specific vector}\GaoC{Okay} 
the kernel of $\widehat{\bL}$, i.e., ${\rm ker}(\widehat{\bL})$ as $t \xrightarrow{} \infty$, resulting in non-separation of nodes with same degrees, known as the over-smoothing issue. 

A recent work \cite{di2022graph} also shows that even with nonlinear activation and weight matrix as in classic GCN \eqref{eq_classic_gcn}, the process is still dominated by the low frequency, thus eventually converging to the kernel of $\widehat{\bL}$, for almost every initialization. To quantify such behavior, \cite{di2022graph} considers a general dynamics as $\dot{\bH}(t) = {\rm GNN}_\theta (\bH(t), t)$, with ${\rm GNN}_\theta(\cdot)$ as an arbitrary graph neural network function, and also characterizes its behavior by low/high-frequency-dominance (L/HFD). 

\begin{definition}[\cite{di2022graph}]
\label{def_hfd_lfd}
$\dot{\bH}(t) = {\rm GNN}_\theta (\bH(t), t)$ is Low-Frequency-Dominant (LFD) if $E \big(\bH(t)/ \| \bH(t) \| \big) \xrightarrow{} 0$ as $t \xrightarrow{} \infty$, and is High-Frequency-Dominant (HFD) if $E \big(\bH(t)/ \| \bH(t) \| \big) \xrightarrow{} \rho_L/2$ as $t \xrightarrow{} \infty$. 
\end{definition}

\begin{lemma}[\cite{di2022graph}]
\label{lemma_hfd_lfd}
A GNN model is LFD (resp. HFD) if and only if for each $t_j \xrightarrow{} \infty$, there exists a subsequence indexed by $t_{j_k} \xrightarrow{} \infty$ and $\bH_{\infty}$ such that $\bH(t_{j_k})/\| \bH(t_{j_k})\| \xrightarrow{} \bH_{\infty}$ and $\widehat{\bL} \bH_{\infty} = 0$ (resp. $\widehat{\bL} \bH_{\infty} = \rho_L \bH_{\infty}$).
% where $\bH_{\infty}$ is $r$-th column of $\bH_{\infty}$. 
% \GaoC{Why write out column individaully?}
% \AHcomment{Removed}
\end{lemma}

\begin{remark}[LFD, HFD, homophilic and heterophilic]
The characterization of LFD and HFD has direct consequences on model performance for homophilic and heterophilic graph datasets. As shown in Lemma \ref{lemma_hfd_lfd}, a GNN model is LFD if and only if there exists a subsequence of $\bH(t)$ that converges to ${\rm ker}(\widehat{\bL})$ (thus smoothing) and is HFD if and only if there exists a subsequence that converges to the eigenvector associated with the largest frequency of $\widehat{\bL}$ (thus separating). In the case of homophilic graph, where neighbouring nodes tend to share the same label, low-frequency information is more important while in contrast for heterophilic graph, high-frequency information is more important. Thus, ideally, a model should be flexible enough to accommodate both LFD and HFD dynamics. 
\end{remark}

\section{Dirichlet energy, gradient flow and framelet convolution}
\label{sect_dirichlet_energy_framelet}

This section studies the following linearized spatial framelet convolution proposed in \cite{chendirichlet} and its connection to energy and gradient flow. 
\begin{equation}
    \bH(\ell + 1) = \gW^\top_{0,J}  \widehat{\bA} \gW_{0,J} \bH(\ell)  \bW_{0,J} + \sum_{r,j} \gW^\top_{r,j} \widehat{\bA} \gW_{r,j} \bH(\ell)  \bW_{r,j}. \label{eq_main_sp_frame_redefine}
\end{equation}
The consideration of the linear dynamics as in \eqref{eq_main_sp_frame_redefine} is driven by the definition of an energy, which is often a quadratic form \cite{di2022graph}. Further, it has been shown in \cite{wang2022powerful,wu2019simplifying} that removing nonlinearity often does not severely harm the prediction performance. Nevertheless, we investigate the effect of nonlinearity on the gradient flow in Section \ref{activation_sect}. Furthermore, we require the weight matrices $\bW_{r,j}^\ell = \bW_{r,j}$ to be fixed across all layers and to be \textit{symmetric} (in order to ensure that it leads to a gradient flow of some energy as discussed in Section \ref{sect_genalized_energy}). The symmetric assumption is not strict, given the universal approximation results in \cite{hu2019exploring} and has also been considered in \cite{di2022graph}.

% For the current section, we focus on the linearized dynamics of \eqref{eq_main_sp_frame_redefine}, i.e., without the activation function, and investigate the effect of activation function in Section \ref{}.

We start with a result showing that the total Dirichlet energy for all framelet coefficients (i.e., $\gW_{r,j} \bH$) is identical to the Dirichlet energy on the spatial domain.

\begin{proposition}[Energy conservation via framelet decomposition \cite{chendirichlet}]
\label{prop_energy_conserve}
For a graph signal $\bH$, let $E_{0,J}(\bH) = \frac{1}{2} \trace \big( (\gW_{0,J} \bH)^\top \widehat{\bL} \gW_{0,J} \bH \big)$ and similarly $E_{r,j}(\bH) = \frac{1}{2} \trace \big( (\gW_{r,j} \bH)^\top \widehat{\bL} \gW_{r,j} \bH \big)$ for $r=1,...,L, j = 1,...,J$. Then we have $E(\bH) = E_{0,J}(\bH) + \sum_{r,j} E_{r,j}(\bH)$. 
\end{proposition}

Proposition \ref{prop_energy_conserve} immediately shows that the gradient flow of the Dirichlet energy can be equivalently expressed as $\dot{\bH}(t) = - \nabla E(\bH(t)) = - \Big( \gW_{0,J}^\top \widehat{\bL} \gW_{0,J} \bH(t) + \sum_{r,j} \gW_{r,j}^\top \widehat{\bL} \gW_{r,j} \bH(t) \Big)$ where its discretization (with a stepsize $\tau = 1$) leads to 
\begin{align*}
    \bH(\ell+1) = \gW_{0,J}^\top \widehat{\bA} \gW_{0,J} \bH(\ell) + \sum_{r,j} \gW_{r,j}^\top \widehat{\bA} \gW_{r,j} \bH(\ell)
\end{align*}
due to $\gW_{0,J}^\top \gW_{0,J} + \sum_{r,j} \gW_{r,j}^\top \gW_{r,j} = \bI_n$. 
Comparing the above dynamics to the framelet convolution in \eqref{eq_main_sp_frame_redefine}, we see the only difference is in the weight matrices $\bW_{r,j}$. In fact, if the weight matrices $\bW_{r,j}$ are the same across all $r,j$, the dynamics would be equivalent to that of GCN, thus inducing a LFD dynamics for almost every $\bH(0)$ \cite[Theorem 4.3]{di2022graph}. To see this, let $\bW_{0,J} = \bW_{r,j} =\bW$. Then \eqref{eq_main_sp_frame_redefine} becomes 
\begin{align*}
    \bH(\ell + 1) &= \big( \gW_{0,J}^\top \widehat{\bA} \gW_{0,J} +  \sum_{r,j} \gW_{r,j}^\top \widehat{\bA} \gW_{r,j} \big) \bH(\ell) \bW \\
    &= \Big( \bU^\top \bLambda_{0,J}^2 (\bI_n - \bLambda) \bU + \sum_{r,j} \bU^\top \bLambda_{r,j}^2 (\bI_n - \bLambda) \bU \Big) \bH(\ell) \bW \\
    &= \bU^\top (\bI_n - \bLambda) \bU  \bH(t) \bW = \widehat{\bA} \bH(\ell) \bW
\end{align*}
where we use the fact that $\bLambda_{0,J}^2 + \sum_{r,j} \bLambda_{r,j}^2 = \bI_n$.
However, we show in Theorem \ref{thm_frame_lfdhfd} that with different $\bW_{r,j}$, the framelet convolution is able to induce a HFD dynamics. 

For the purpose of analysis, we focus on the Haar-type filter \cite{dong2017sparse}, which is commonly considered for framelet convolution, such as \cite{chendirichlet,zheng2021framelets,zhou2021graph,zhou2021spectral}. We remark that similar analysis can be performed for other filters. 
The filter bank of Haar filter is given by a low-pass filter $\widehat{a}(\xi) = \cos(\xi/2)$, and a high-pass filter $\widehat{b}(\xi) = \sin(\xi/2)$. We focus on the cases when the scale is $J =1$ and $J = 2$, which are sufficient for practical applications \cite{chendirichlet,zheng2021framelets}.
% \GaoC{Citation here} 
When $J = 1$, we can explicitly write the induced scaling functions as $\widehat{\alpha}(\bLambda/2) = \cos(\bLambda/8)$ and $\widehat{\beta}(\bLambda/2) = \sin(\bLambda/8)$. When $J = 2$, the scaling functions are given in \cite[Proposition 6]{chendirichlet} as $\widehat{\alpha}(\bLambda/2) = \cos^2(\bLambda/8) \cos(\bLambda/16)$, $\widehat{\beta}(\bLambda/2) = \sin(\bLambda/8) \cos(\bLambda/16)$, $\widehat{\beta}(\bLambda/4) = \sin(\bLambda/16)$.
% \begin{equation*}
%     \widehat{\alpha}(\bLambda/2) = \cos^2(\bLambda/8) \cos(\bLambda/16), \quad \widehat{\beta}(\bLambda/2) = \sin(\bLambda/8) \cos(\bLambda/16), \quad \widehat{\beta}(\bLambda/4) = \sin(\bLambda/16).
% \end{equation*}

\begin{theorem}
\label{thm_frame_lfdhfd}
The (spatial) graph framelet convolution \eqref{eq_main_sp_frame_redefine} with Haar-type filter can induce both LFD and HFD dynamics. Specifically, let $\bW_{0, J} = \bI_c, \bW_{r,j} = \lambda^W \bI_c$ for $r=1,...,L, j = 1,...,J$. Then the framelet convolution can be HFD when $|\lambda^W| > 1$ and is sufficiently large. On the other hand, when $|\lambda^W| \leq 1$, the framelet convolution can only be LFD. 
\end{theorem}
\begin{proof}
Consider the linearized framelet convolution \eqref{eq_main_sp_frame_redefine} with a stepsize $\tau$ as 
% \GaoC{This can be regarded a new way? Seems no one considers this.} \AHcomment{correct. this is a linear framelet model, similar to linear GCN.}
\begin{equation}
    \bH(t + \tau) = \tau \gW_{0,J}^\top \widehat{\bA} \gW_{0,J} \bH(t) \bW_{0, J} + \tau \sum_{r,j} \gW_{r,j}^\top \widehat{\bA} \gW_{r,j} \bH(t) \bW_{r,j}, \label{eq_discrete_frame}
\end{equation}
for some symmetric matrices $\bW_{0,J}, \bW_{r,j} \in \sR^{c \times c}$ where $c$ is the number of channels (feature dimension). 
% First, it is easy to verify that \eqref{eq_discrete_frame} can induce a LFD dynamics by simply setting $\bW_{0,J} = \bW_{r,j} = \bW, \forall r,j$. This is because
% \begin{align*}
%     \gW_{0,J}^\top \widehat{\bA} \gW_{0,J} + \sum_{r,j} \gW_{r,j}^\top \widehat{\bA} \gW_{r,j} = \bU^\top \bLambda_{0,J}^2 (\bI_n - \bLambda) \bU + \sum_{r,j} \bU^\top \bLambda_{r,j}^2 (\bI_n - \bLambda) \bU = \bU^\top (\bI_n - \bLambda) \bU = \widehat{\bA},
% \end{align*}
% where we use the fact that $\bLambda_{0,J}^2 + \sum_{r,j} \bLambda_{r,j}^2 = \bI_n$. In this case, \eqref{eq_discrete_frame} is equivalent to GCN, with propagation $\bH(t + \tau) = \tau \widehat{\bA} \bH(t) \bW$ and from \cite[Theorem 4.3]{di2022graph}, we can see the dynamics \eqref{eq_discrete_frame} is LFD for almost every initialization $\bH(0)$. 

We consider the Haar-type filter with one high pass ($r = 1$) as in \cite{dong2017sparse}. The Haar filter with one high pass has filter bank given by $\widehat{a}(\xi) = \cos(\xi/2), \widehat{b}(\xi) = \sin(\xi/2)$. We start the analysis with a single scale $J =1$ and then extend the analysis to two scales $J =2$, as commonly used for framelet convolution in \cite{chendirichlet,zheng2021framelets}. 

\paragraph{Haar filter with one scale.} 
When $J = 1$, 
we have 
\begin{equation*}
    \gW_{0,1} = \bU^\top \bLambda_{0,1} \bU = \bU^\top \cos(\bLambda/8) \bU, \quad \gW_{1,1} = \bU^\top \bLambda_{1,1} \bU = \bU^\top \sin(\bLambda/8) \bU,
\end{equation*}
% with $\widehat{\alpha}(\cdot)$ an non-increasing function and $\widehat{\beta}(\cdot)$ an non-decreasing function satisfying $\widehat{\alpha}^2(\xi) + \widehat{\beta}^2(\xi) = 1$, for all $\xi \in \sR$. Without loss of generality, we assume $\widehat{\alpha}(0) = 1, \widehat{\beta}(0) = 0$, which is the case for Haar-type filter 
% From \cite[Proposition 6]{chendirichlet},  we have 
% \begin{align*}
%     \gW_{0,1} &= \bU^\top \widehat{\alpha}(\bLambda/2) \bU = \bU^\top \cos(\bLambda/8) \cos(\bLambda/16) \bU = \bU^\top \bLambda_{0,2} \bU, \\
%     \gW_{1,1} &= \bU^\top \widehat{\beta}(\bLambda/2) \bU = \bU^\top \sin(\bLambda/8) \cos(\bLambda/16) \bU = \bU^\top \bLambda_{1,1} \bU, \\
%     \gW_{1,2} &= \bU^\top \widehat{\beta}(\bLambda/4) \bU = \bU^\top \sin(\bLambda/16) \bU = \bU^\top \bLambda_{1,2} \bU,
% \end{align*}
Now consider the vectorization of \eqref{eq_discrete_frame} as
\begin{equation}
    \vec \big( \bH(m\tau) \big) = \Big( \tau \bW_{0,1} \otimes \gW_{0,1}^\top \widehat{\bA} \gW_{0,1}  + \tau \bW_{1,1} \otimes \gW_{1,1}^\top \widehat{\bA} \gW_{1,1} \Big)^m \vec \big( \bH(0) \big), \label{eq_vec_pro}
\end{equation}
where $\otimes$ represents the Kronecker product.
To simplify the analysis, we assume $\bW_{0,1} = \bI_{c}, \bW_{1,1} = \lambda^W \bI_c$ without loss of generality. From standard results on Kronecker product, we can rewrite \eqref{eq_vec_pro} as
\begin{equation}
    \vec \big( \bH(m\tau) \big) = \tau^m \sum_{k,i} \Big( \big( (\lambda_{i}^{\Lambda_{0,1}})^2 +  \lambda^W (\lambda_{i}^{\Lambda_{1,1}})^2  \big) (1 - \lambda_i) \Big)^m c_{k,i}(0) \be_k \otimes \bu_i, \label{eq_vec_prop}
\end{equation}
where we recall that $\{(\lambda_i, \bu_i)\}$ are the eigen-pairs of $\widehat{\bL}$ and $c_{k,i}(0) = \langle \vec(\bH(0)), \be_k \otimes \bu_i \rangle$. We also denote $\lambda_i^{\Lambda_{r,j}} = (\bLambda_{r,j})_{ii}$ for all $(r,j)$.  We show when $|\lambda^{W}| > 1$, the dynamics \eqref{eq_vec_prop} can be HFD.\footnote{We note that there may exist other settings where the dynamics is HFD. Here we only provide one case where this happens.} 

When $\lambda_i = 0$, we have $\left\vert \big(  (\lambda_i^{\Lambda_{0,1}})^2 + \lambda_k^{W} (\lambda_{i}^{\Lambda_{1,1}})^2 \big) (1-\lambda_i) \right\vert = 1$.

When $\lambda_i \neq 0$, we have 
\begin{align*}
    \left\vert \big(  (\lambda_i^{\Lambda_{0,1}})^2 + \lambda^{W} (\lambda_{i}^{\Lambda_{1,1}})^2 \big) (1 - \lambda_i)  \right\vert &=  \left\vert \Big( \cos^2(\lambda_i/8) + \lambda^{W} \sin^2(\lambda_i/8) \Big) (1 - \lambda_i) \right\vert \\
    &=  \left\vert \cos^2(\lambda_i/8) + \lambda^{W} \sin^2(\lambda_i/8)  \right\vert | 1 - \lambda_i |.
\end{align*}
First given that $\lambda_i \in [0,\rho_L],$ $\rho_L \leq 2$, we can see when $|\lambda^W| \leq 1$, $\cos^2(\lambda_i/8) + \lambda^W \sin^2(\lambda_i/8) \in (0,1]$ and is non-increasing in $\lambda_i$. Hence $\left\vert \big(  (\lambda_i^{\Lambda_{0,1}})^2 + \lambda_k^{W} (\lambda_{i}^{\Lambda_{1,1}})^2 \big) (1-\lambda_i) \right\vert$ takes it maximum at $\lambda_i = 0$ regardless of the value of $\lambda^W$. This suggests the dominant frequency is $0$, thus inducing a LFD dynamics. 

On the other hand, when $|\lambda^W| > 1$ and is sufficiently large (with respect to the $\rho_L$), then we can show the function $\left\vert \big(  (\lambda_i^{\Lambda_{0,1}})^2 + \lambda^{W} (\lambda_{i}^{\Lambda_{1,1}})^2 \big) (1-\lambda_i) \right\vert$ takes its \textit{unique} maximum at frequency $\rho_L$.\footnote{We remark that when $\rho_L \xrightarrow{} 1$, for sufficiently large $|\lambda^W| > 1$, it can happen that the maximum is attained at frequency $\lambda \in (0,1)$. In this case, the dynamics is neither LFD nor HFD.} Particularly, denote $\delta_{\rm HFD} := \big(\cos^2(\rho_L/4) + \lambda_{+}^{W} \sin^2(\rho_L/4) \big) (\rho_L -1)$ and
suppose $|\lambda^W| > 1$ sufficiently large such that for all $i$ where $\lambda_i \neq \rho_L$, $\vert \big(  (\lambda_i^{\Lambda_{0,1}})^2 + \lambda_k^{W} (\lambda_{i}^{\Lambda_{1,1}})^2 \big) (1 - \lambda_i) \vert < \delta_{\rm HFD}$ holds. Then we can show the dynamics is HFD where the dominant frequency is $\rho_L$. 

More precisely, let $\delta := \max_{i: \lambda_i \neq \rho_L} \vert \big(  (\lambda_i^{\Lambda_{0,1}})^2 + \lambda^{W} (\lambda_{i}^{\Lambda_{1,1}})^2 \big) (1 - \lambda_i) \vert$. We also denote $\bP_\rho = \sum_{k} (\be_{k} \otimes \bu_\rho)(\be_{k} \otimes \bu_\rho)^\top$ where $\bu_\rho$ is the eigenvector of $\widehat{\bL}$ associated with eigenvalue $\rho_L$ (assuming the eigenvalue $\rho_L$ is simple). Then we can decompose \eqref{eq_vec_prop} as 
% It can be shown that for $a > 1$ sufficiently large, the function $g(\lambda) = (1 + a \sin^2(\lambda/8)) |1 - \lambda|$ takes its \textit{unique} maximum at $\lambda = \rho_L$ for $\lambda \in (0, \rho_L]$. Particularly, suppose $\lambda_{+}^W$ sufficiently large (with respect to the $\rho_L$) such that 
% \begin{equation*}
%     \left\vert \big(  (\lambda_i^{\Lambda_{0,1}})^2 + \lambda_k^{W} (\lambda_{i}^{\Lambda_{1,1}})^2 \big) (1 - \lambda_i)  \right\vert < \Big(\cos^2(\rho_L/4) + \lambda_{+}^{W} \sin^2(\rho_L/4) \Big) (\rho_L -1),
% \end{equation*}
% Let $\delta_{\rm HFD} := \Big(\cos^2(\rho_L/4) + \lambda_{+}^{W} \sin^2(\rho_L/4) \Big) (\rho_L -1)$ and suppose $\delta_{\rm HFD} > 1$ for large $\lambda_{+}^W$. Then the dominant frequency is $\lambda = \rho_L$, thus inducing a HFD dynamics. 
\begin{align*}
    &\vec\big( \bH(m\tau) \big) \\
    &= \tau^m \sum_{k } \delta_{\rm HFD}^m  c_{k, \rho_L}(0) \bu_k \otimes \bu_\rho + \tau^m \sum_{k} \sum_{i : \lambda_i \neq \rho_L} \Big( \big( (\lambda_{i}^{\Lambda_{0,1}})^2 +  \lambda^{W} (\lambda_{i}^{\Lambda_{1,1}})^2 \big) (1-\lambda_i) \Big)^m c_{k, i}(0) \be_k \otimes \bu_\rho \\
    &\leq \tau^m \delta_{\rm HFD}^m \left(  \bP_\rho \vec \big( \bH(0) \big) + \sum_{k} \sum_{i : \lambda_i \neq \rho_L}  \left( \frac{ \delta }{\delta_{\rm HFD}} \right)^m  c_{k, i}(0) \be_k \otimes \bu_\rho \right),
\end{align*}
where $\delta < \delta_{\rm HFD}$.
By normalizing the results, we obtain $\frac{\vec\big( \bH(m\tau) \big)}{\|  \vec\big( \bH(m\tau) \big)\|} \xrightarrow{} \frac{\bP_\rho (\vec(\bH(0)))}{\| \bP_\rho \vec (\bH(0)) \|}$, as $m \xrightarrow{} \infty$, where the latter is a unit vector $\bh_{\infty}$ satisfying $(\bI_c \otimes \widehat{\bL}) \bh_{\infty} = \rho_L \bh_{\infty}$. This suggests, the dynamics is HFD according to Definition \ref{def_hfd_lfd} and Lemma \ref{lemma_hfd_lfd}.

\paragraph{Haar filter with two scales.}
When $J = 2$, we have
\begin{equation*}
    \gW_{0,2} = \bU^\top \cos(\bLambda/8) \cos(\bLambda/16) \bU, \quad \gW_{1,1} = \bU^\top \sin(\bLambda/8) \cos(\bLambda/16) \bU, \quad \gW_{1,2} = \bU^\top \sin(\bLambda/16) \bU.
\end{equation*}
Following the same strategy, we assume $\bW_{0,2} = \bI_c$ and $\bW_{1,1} = \bW_{1,2} = \lambda^W \bI_c$. 
Then the vectorization of the dynamics writes
\begin{equation*}
    \vec\big( \bH(m \tau) \big) = \tau^m \sum_{k, i} \Big(  (\lambda_i^{\Lambda_{0,2}})^2 + \lambda^W  \big( (\lambda_i^{\Lambda_{1,1}})^2 + (\lambda_i^{\Lambda_{1,2}})^2 \big)  \Big)^m  (1- \lambda_i)^m c_{k,i}(0) \be_k \otimes \bu_i.
\end{equation*}
It can be similarly verified that when $|\lambda^W| < 1$, we have $(\lambda_i^{\Lambda_{0,2}})^2 + \lambda^W  \big( (\lambda_i^{\Lambda_{1,1}})^2 + (\lambda_i^{\Lambda_{1,2}})^2 \big) \in (0,1]$ and is non-increasing in $\lambda_i$. Thus in this case, the dynamics can only be LFD for the same reason. On the other hand, for $|\lambda^W| > 1$ sufficiently large, the dominant frequency can be $\rho_L$ and the dynamics is HFD given the definitions of the scaling functions above. Following the same step as above, we can show the dynamics is HFD.
% When $\lambda_i = 0$, we have $\vert \big( (\lambda_i^{\Lambda_{0,2}})^2 + \lambda_k^W  \big( (\lambda_i^{\Lambda_{1,1}})^2 + (\lambda_i^{\Lambda_{1,2}})^2 \big) \big) (1- \lambda_i) \vert = 1$. 
% When $\lambda_i \neq 0$, we have 
% \begin{align}
%     \left\vert \Big( (\lambda_i^{\Lambda_{0,2}})^2 + \lambda_k^W  \big( (\lambda_i^{\Lambda_{1,1}})^2 + (\lambda_i^{\Lambda_{1,2}})^2 \big) \Big) (1- \lambda_i) \right\vert &= \left\vert 1 + \big( \lambda_k^W - 1 \big) \big( (\lambda_i^{\Lambda_{1,1}})^2 + (\lambda_i^{\Lambda_{1,2}})^2 \big) \right\vert |1- \lambda_i| \nonumber\\
%     &\leq \left( 1 + \big( \lambda_+^W - 1 \big) \big( (\lambda_i^{\Lambda_{1,1}})^2 + (\lambda_i^{\Lambda_{1,2}})^2 \big) \right) |1- \lambda_i| \label{eq_temp}
% \end{align}
% where $(\lambda_i^{\Lambda_{1,1}})^2 + (\lambda_i^{\Lambda_{1,2}})^2 = \sin^2(\lambda_i/8) \cos^2(\lambda_i/16) + \sin^2(\lambda_i/16)$. Let 
% $$\delta_{\rm HFD} := \Big(1 + (\lambda_+^W -1) \big( \sin^2(\rho_L/8) \cos^2(\rho_L/16) + \sin^2(\rho_L/16) \big) \Big) (\rho_L -1).$$ 
% From similar analysis, we can also show when $\lambda_+^W > 1$ sufficiently large, \eqref{eq_temp} attains the unique maximum at $\delta_{\rm HFD}$ with $\lambda_i = \rho_L$ and it satisfies that $\delta_{\rm HFD} > 1$. Following the same steps as for the single-scale case, the dynamics is HFD.
\end{proof}

% Finally, we remark that the conditions on $|\lambda^W| > 1$ and $|\lambda^W|$ being sufficiently large are intuitive given the eigenvalues of weight matrix for the low-pass filter is $1$. And more importantly, we allow all the weight matrices to be positive definite, which is not the case without framelet convolution.

\begin{remark}
Importantly from Theorem \ref{thm_frame_lfdhfd}, we conclude that given $\bW_{0,J} = \bI_c$, $\bW_{r,j} = \lambda^W \bI_c$, the framelet convolution can be HFD when $|\lambda^W| > 1$ and is sufficiently large, i.e., when the eigenvalue magnitude of the high-pass weight matrices is sufficiently larger than that of the low-pass. We also show that when $|\lambda^W| < 1$, the dynamics can only be LFD. This conclusion is irrespective of whether there exists any negative spectrum in the weight matrices (i.e., in this case, $\lambda^W$ can be negative or positive). 
In contrast, without framelet decomposition, in order to induce a HFD dynamics, \cite{di2022graph} requires a sufficiently large negative spectrum of the weight matrix. This shows with framelet decomposition, the parameters are given more flexibility in adapting to the right frequency. 
\end{remark}

Based on Theorem \ref{thm_frame_lfdhfd}, we shall expect framelet convolution to perform better than GCN on both homophilic and heterophilic graphs, which aligns with empirical observations \cite{chendirichlet}.

\section{Generalized framelet energy and its gradient flow}
\label{sect_genalized_energy}

This section proposes a generalized energy via framelet decomposition. 
We first recall the general energy $\gE(\bH)$ proposed in \cite{di2022graph}, which is designed to induce both LFD and HFD dynamics. That is, $\gE(\bH) = \frac{1}{2} \sum_i \langle \bh_i, \bOmega \bh_i \rangle - \frac{1}{2} \sum_{i,j} {a}_{ij} \langle \bh_i, \bW \bf_j\rangle$, for some symmetric matrices $\bOmega, \bW$, and can be written alternatively as
\begin{align}
    \gE(\bH) = \frac{1}{2} \trace \big(\bH^\top \bH \bOmega \big) - \frac{1}{2} \trace\big( \bH^\top \widehat{\bA} \bH \bW \big) = \frac{1}{2} \langle \vec( \bH ), (\bOmega \otimes \bI_n - \bW \otimes \widehat{\bA}) \vec(\bH) \rangle \label{eq_gen_energy}
\end{align}
It is worth noticing that the generalized energy includes the Dirichlet energy as a special case. By setting $\bOmega = \bW = \bI_c$, we recover the Dirichlet energy. It has been shown in \cite{di2022graph} that the positive spectrum of $\bW$ in \eqref{eq_gen_energy} 
% \GaoC{What is this notation?} 
encodes \textit{attractive interactions} via graph gradients while the negative spectrum encodes \textit{repulsive interactions}. Thus, by adjusting the spectrum of $\bW$ provides a control for frequency reduction and amplification. 

Now we propose a framelet generalized energy as follows. Let $\gE_{r,j}(\bH) = \frac{1}{2} \trace\big( (\gW_{r,j} \bH)^\top \gW_{r,j} \bH \bOmega_{r,j} \big) - \frac{1}{2} \trace \big( (\gW_{r,j} \bH)^\top \widehat{\bA} \gW_{r,j} \bH \bW_{r,j} \big)$ for all $(r,j) \in \gI$. Then the \textit{total framelet generalized energy} is given by
\begin{align*}
    \gE^{\rm tot}(\bH) &= \gE_{0, J}(\bH) + \sum_{r,j} \gE_{r,j}(\bH) \\
    &= \frac{1}{2} \sum_{(r,j) \in \gI} \left\langle \vec(\bH) , \Big( \bOmega_{r,j} \otimes \gW_{r,j}^\top \gW_{r,j} - \bW_{r,j} \otimes \gW_{r,j}^\top \widehat{\bA} \gW_{r,j} \Big)   \vec(\bH)  \right\rangle,
\end{align*} 
where we denote the index set $\gI = \{(r,j) : r = 1,...,L, j = 1,...,J \} \cup \{ (0, J) \}$. In order for $\gE^{\rm tot}(\bH)$ to be an energy, we require $\nabla^2 \gE^{\rm tot}(\bH) = \sum_{(r,j) \in \gI} \bOmega_{r,j} \otimes \gW_{r,j}^\top \gW_{r,j} - \bW_{r,j} \otimes \gW_{r,j}^\top \widehat{\bA} \gW_{r,j} $ to be symmetric,
% \GaoC{and Positive?} \AHcomment{only symmetric is required.} 
which can be satisfied by requiring $\bOmega_{r,j}, \bW_{r,j}$ to be symmetric.

Notice that $\gE^{\rm tot}(\bH)$ extends the energy $\gE(\bH)$ proposed in \cite{di2022graph} without framelet decomposition \eqref{eq_gen_energy}, which we formalize in the below proposition.

\begin{proposition}
\label{prop_gen_energy_conserve}
When $\bOmega_{r,j} = \bOmega, \bW_{r,j} = \bW$, we have $\gE^{\rm tot}(\bH) = \gE(\bH)$.
\end{proposition}
\begin{proof}
The proof follows again from the tightness of the framelets.
\begin{align*}
    \gE^{\rm tot}(\bH) &= \gE_{0, J}(\bH) + \sum_{r,j} \gE_{r,j}(\bH) \\
    &= \frac{1}{2} \trace \big( \bH^\top \gW_{0, J}^\top \gW_{0,J} \bH \bOmega \big) - \frac{1}{2} \trace \big( \bH^\top \gW_{0,J}^\top \widehat{\bA} \gW_{0,J} \bH \bW \big) \\
    &\qquad + \sum_{r,j} \Big( \frac{1}{2} \trace \big( \bH^\top \gW_{r, j}^\top \gW_{r,j} \bH \bOmega \big) - \frac{1}{2} \trace\big( \bH^\top \gW_{r,j}^\top \widehat{\bA} \gW_{r,j} \bF \bW \big) \Big) \\
    &= \frac{1}{2} \trace \big( \bH^\top \bH \bOmega \big) - \frac{1}{2} \trace \Big( \bU^\top  \bLambda_{0, J}^2 (\bI_n - \bLambda) \bU \bH \bW\bH^\top \Big) - \frac{1}{2} \sum_{r,j} \trace \Big( \bU^\top \bLambda_{r,j}^2 (\bI_n - \bLambda) \bU \bH \bW \bH^\top \Big) \\
    &= \frac{1}{2} \trace \big( \bH^\top \bH \bOmega \big) - \frac{1}{2} \trace \big( \bH^\top \widehat{\bA} \bH \bW \big) = \gE(\bH),
\end{align*}
where we use the fact that $\bLambda_{0,J}^2 + \sum_{r,j} \bLambda_{r,j}^2 = \bI_n$.
\end{proof}

\subsection{Proposed gradient flow of the generalized energy}
We next propose a GNN model, named gradient flow based undecimated framelet convolution (GradF-UFG), which is given by the discretization of the gradient flow of the proposed framelet energy $\gE^{\rm tot}$, i.e. $\dot{\bH}(t) = - \nabla \gE^{\rm tot}(\bH(t))$ as
\begin{align}
 \quad {\bH}(t + \tau) 
    = \bH(t) - \tau \sum_{(r,j) \in \gI} \Big( \gW_{r,j}^\top \gW_{r,j} \bH(t) \bOmega_{r,j} - \gW_{r,j}^\top \widehat{\bA}  \gW_{r,j} \bH(t) \bW_{r,j} \Big) \label{eq_gradf_ufg}
\end{align}
with $\bH(0) = \mathtt{MLP}_\theta(\bX)$, where $\mathtt{MLP}_\theta(\cdot)$ is an encoder and $\bX$ is the input feature matrix. The encoder model can be used to alter the dimension of the input features and also include nonlinearity to increase the expressiveness. In Section \ref{activation_sect}, we explore an alternative way to include nonlinearity and study its influence on the energy.

Next we show in the below proposition that the framelet convolution \eqref{eq_main_sp_frame_redefine} is in fact a special instance of the proposed gradient flow discretization, GradF-UFG.

\begin{proposition}[Framelet convolution as gradient flow of the proposed generalized energy $\gE^{\rm tot}$]
\label{prop_frame_gradient_flow}
When $\bOmega_{0, J} = \bOmega_{r,j} = \bI_c$, then framelet convolution \eqref{eq_main_sp_frame_redefine} corresponds to the proposed GradF-UFG with a stepsize $\tau = 1$.
\end{proposition}
\begin{proof}
The proof directly follows from the fact that $\gW_{0,J}^\top \gW_{0,J} + \sum_{r,j} \gW_{r,j}^\top \gW_{r,j} = \bI_n$.
% \begin{align*}
%     \dot{\bH}(t) = - \nabla \gE^{\rm tot}(\bH(t)) &= - \gW_{0,J}^\top \gW_{0, J} \bH(t) - \sum_{r,j} \gW_{r,j}^\top \gW_{r,j} \bH(t) \\
%     &\qquad + \gW_{0,J}^\top \widehat{\bA} \gW_{0,J} \bH(t) \bW_{0,J} + \sum_{r,j} \gW_{r,j}^\top \widehat{\bA} \gW_{r,j} \bH(t) \bW_{r,j} \\
%     &= - \bH(t) + \gW_{0,J}^\top \widehat{\bA} \gW_{0,J} \bH(t) \bW_{0,J} + \sum_{r,j} \gW_{r,j}^\top \widehat{\bA} \gW_{r,j} \bH(t) \bW_{r,j}.
% \end{align*}
% Its discretization with a stepsize $\tau = 1$ completes the proof.
\end{proof}

\paragraph{A multi-particle perspective on the proposed energy and its gradient flow.}
We can follow \cite{di2022graph} by providing a multi-particle interpretation for the proposed generalized framelet energy $\gE^{\rm tot}$ as well as its gradient flow in \eqref{eq_gradf_ufg}. We may view each node as a particle in $\sR^c$ with energy $\gE^{\rm tot}$. Let $\bTheta^{+}_{r,j}, \bTheta^{-}_{r,j} \in \sR^{c \times c}$ be defined such that $\bW_{r,j} = (\bTheta^{+}_{r,j})^\top \bTheta^{+}_{r,j} - (\bTheta^{-}_{r,j})^\top \bTheta^{-}_{r,j}$ for all $(r,j) \in \gI$.
% \GaoC{Can we do this for the Laplacian, then framelet can be implemented on the postive and negative parts individually?} \AHcomment{Laplacian is positive semidefinite. So only positive parts. Or do you mean adjacency?} \GaoC{Yes. I was thinking about digraph Laplacian, but it is not symmetric.} 
In other words, the weight matrices can be decomposed into a positive definite and a negative definite component. Then the energy $\gE^{\rm tot}$ can be rewritten as 
\begin{align}
    \gE^{\rm tot}(\bH) &= \sum_{(r,j) \in \gI} \Big( \frac{1}{2} \sum_{i =1}^n  \big\langle \gW_{r,j} \bh_i, (\bOmega_{r,j} - \bW_{r,j}) \gW_{r,j} \bh_i \big\rangle \nonumber\\
    &\qquad \qquad + \frac{1}{4} \sum_{i,i'=1}^n \| \bTheta^+_{r,j} \nabla (\gW_{r,j} \bH)_{ii'} \|^2 - \frac{1}{4} \sum_{i,i'=1}^n \| \bTheta^{-}_{r,j} \nabla (\gW_{r,j} \bH)_{ii'} \|^2  \Big), \label{eq_energy_decompose}
\end{align}
where $\nabla (\gW_{r,j} \bH)_{ii'} = \gW_{r,j} (\bh_i/\sqrt{d_i} - \bh_{i'}/\sqrt{d_{i'}})$ is the graph gradient along edge $(i,i') \in \gE_\gG$ in terms of the framelet coefficients. The first term in \eqref{eq_energy_decompose} corresponds to an external energy, independent of the particle interaction. The latter two terms in \eqref{eq_energy_decompose} encode the attraction and repulsion respectively along the graph gradients projected onto the framelet domain.  

The gradient flow of such energy \eqref{eq_gradf_ufg} thus minimizes the divergence between the projected gradients through the positive spectrum of $\bW_{r,j}$ and maximizes the divergence through the negative spectrum. It should be noticed that when $\bOmega_{r,j} = \bOmega, \bW_{r,j} = \bW$,
% \GaoC{double check this statement here.} 
$\gE^{\rm tot}$ reduces to the energy $\gE$ without framelet decomposition in \cite{di2022graph}, as shown in Proposition \ref{prop_gen_energy_conserve}. In this case, let $\bW = (\bTheta^{+})^\top \bTheta^{+} - (\bTheta^{-})^\top \bTheta^{-}$. Then minimizing $\| \bTheta^+ (\nabla \bH)_{ij} \|^2$ introduces a smoothing effect and maximizing $\| \bTheta^- (\nabla \bH)_{ij} \|^2$ induces a separating effect to the system. This suggests a sufficiently large negative spectrum (compared to the positive spectrum) is able to yield a HFD dynamics. 

In contrast, we have shown in Theorem \ref{thm_frame_lfdhfd} that with framelet decomposition, all $\bW_{r,j}$ can be positive definite for a HFD system as long as sufficient weights are put on the high-pass components. In addition, defining energy on different framelets separately in \eqref{eq_energy_decompose} allows effects other than smoothing or separating (e.g. via band-pass filters), which can lead to a more flexible dynamics.

\subsection{Discussions and comparisons with framelet enhanced energy}

A recent work \cite{chendirichlet} introduces the following framelet perturbed energy $E^\epsilon(\bH)$, which provably enhances the Dirichlet energy for the same input signal.
\begin{equation*}
    E^{\epsilon}(\bH) = \frac{1}{2} \trace \big( (\gW_{0,J} \bH)^\top (\widehat{\bL} + \epsilon \bI_n) \gW_{0,J} \bH  \big) + \frac{1}{2} \sum_{r,j} \trace \big( (\gW_{r,j} \bH)^\top (\widehat{\bL} - \epsilon \bI_n ) \gW_{r,j} \bH \big),
\end{equation*}
for some chosen $\epsilon > 0$.
Specifically, $E^\epsilon(\bH)$ slightly decreases the energy for the low-pass while increases the energy for the high-passes. 
The motivation for such definition of energy is that for the \textit{Haar-type filter}, one can show $E^\epsilon(\bH) > E(\bH)$ for any $\epsilon > 0$ due to the energy gap between the low-pass and high-pass filters in this case. However, we highlight that this strategy only applies to Haar-type filter and can be difficult to generalize to arbitrary framelet filters.

Now we first analyze the gradient flow of such energy without weight matrices as 
\begin{align}
    \dot{\bH}(t) = - \nabla E^\epsilon(\bH(t)) &= -\gW_{0,J}^\top (\widehat{\bL} + \epsilon \bI_n) \gW_{0, J} \bH(t) - \sum_{r,j} \gW_{r,j}^\top (\widehat{\bL} - \epsilon \bI_n) \gW_{r,j} \bH(t) \nonumber\\
    &= - \widehat{\bL} \bH(t) - \epsilon \big( \gW_{0,J}^\top \gW_{0,J} - \sum_{r,j} \gW_{r,j}^\top \gW_{r,j}  \big) \bH(t) \nonumber\\
    &= - \bU^\top \big( \bLambda + \epsilon \bDelta_{\rm gap} \big) \bU \bH(t), \label{eq_grad_flow_perturbed}
\end{align}
where we let $\gW_{0,J}^\top \gW_{0,J} - \sum_{r,j} \gW_{r,j}^\top \gW_{r,j} = \bU^\top \bDelta_{\rm gap} \bU$. Particularly for the case of Haar-type filter (with two scales), \cite[Proposition 7]{chendirichlet} shows $\bDelta_{\rm gap} \geq 0$ is a diagonal matrix with nonnegative diagonal entries. This is the primary reason for enhancement compared to the Dirichlet energy. 
Hence for the following analysis, we also focus on such choice of filter, i.e., $\widehat{\alpha}(\bLambda/2) = \cos^2(\bLambda/8) \cos(\bLambda/16)$, $\widehat{\beta}(\bLambda/2) = \sin(\bLambda/8) \cos(\bLambda/16)$, $\widehat{\beta}(\bLambda/4) = \sin(\bLambda/16)$.

We can write the solution of \eqref{eq_grad_flow_perturbed} in vectorized form as  
\begin{equation}
    \vec( \bH(t) ) = \sum_{k,i} e^{- (\lambda_i + \epsilon \, \delta_i) t} c_{k,i}(0) \be_k \otimes \bu_i, \quad c_{k,i}(0) = \left\langle \vec(\bH(0)), \be_k \otimes \bu_i \right\rangle \label{eq_sol_perturb_cont}
\end{equation}
where we denote $\delta_i = (\bDelta_{\rm gap})_i = \cos^2(\lambda_i/8) \cos^2(\lambda_i/16) - \sin^2(\lambda_i/8) \cos^2(\lambda_i/16) - \sin^2(\lambda_i/16)$ as the energy gap at frequency $\lambda_i$, and $\{\be_k\}_{k=1}^c$ are the standard basis vectors of $\sR^c$. Immediately from the solution \eqref{eq_sol_perturb_cont}, we see the dynamics is LFD regardless of the choice of $\epsilon$. This claim is emphasized in the following proposition.

\begin{proposition}
\label{prop_smoothing_perturb}
For any $\epsilon > 0$, the solution $\bH(t)$ from the gradient flow of $E^\epsilon(\bH)$, i.e., $\dot{\bH}(t) = - \nabla E^\epsilon(\bH(t))$ is LFD with $E(\bH(t)) \xrightarrow{} 0$ as $t \xrightarrow{} \infty$.
\end{proposition}
\begin{proof}
The Dirichlet energy of $\bH(t)$ is given by 
\begin{align*}
    E(\bH(t)) = \frac{1}{2} \left\langle \vec \big( \bH(t) \big), (\bI_c \otimes \widehat{\bL}) \vec \big( \bH(t) \big) \right\rangle  &= \frac{1}{2} \sum_{k, i} e^{-2t (\lambda_i + \epsilon \delta_i)} c_{k,i}^2(0) \lambda_i \\
    &\leq \frac{\rho_L}{2} e^{-2t \vartheta}  \| \bH(0) \|^2 
\end{align*}
where we denote $\vartheta := \min_{i : \lambda_i \in (0,\rho_L]} \lambda_i + \epsilon \delta_i > 0$. Hence as $t \xrightarrow{} \infty$, we see $E(\bH(t)) \xrightarrow{} 0$.
\end{proof}

\begin{remark}
Proposition \ref{prop_smoothing_perturb} shows that despite enhancing the energy at each discrete timestamp, i.e., $E^\epsilon(\bH(t)) > E(\bH(t))$, in the limit ($t \xrightarrow{} \infty$), the dynamics still converges to ${\rm ker}(\widehat{\bL})$, resulting in over-smoothing. Nevertheless, the convergence rate to the kernel is indeed slower once we inject the perturbation $\epsilon\, \bDelta_{\rm gap} > 0$.
\end{remark}

In order for the dynamics to be HFD, separate weight matrices are needed for each filter, as done in \cite{chendirichlet}. Specifically, the model proposed in \cite{chendirichlet}, namely energy enhanced undecimated  framelet graph convolution (EE-UFG), is given by
\begin{equation}
    \bH(\ell+1) = \gW_{0,J}^\top \sigma \Big(  (\widehat{\bA} - \epsilon \bI_n) \gW_{0,J} \bH(\ell) \bW_{0,J}^\ell \Big) + \sum_{r,j} \gW_{r,j}^\top \sigma \Big( (\widehat{\bA} + \epsilon \bI_n) \gW_{r,j} \bH(\ell) \bW_{r,j}^\ell \Big) \label{eq_ee_ufg}
\end{equation}
We show in the next proposition that EE-UFG \eqref{eq_ee_ufg} also can be written as the discretization of gradient flow of the proposed generalized framelet energy $\gE^{\rm tot}(\bH)$, thus a special case of the proposed GradF-UFG. 

\begin{proposition}[EE-UFG as gradient flow of the proposed generalized energy $\gE^{\rm tot}$]
\label{prop_ee_conv_gradient_flow}
When $\bOmega_{0,J} = \bI_c + \epsilon \bW_{0,J}$, $\bOmega_{r,j} = \bI_c - \epsilon \bW_{r,j}$, for $r=1,...,L, j = 1,...J$, then EE-UFG (without activation function) corresponds to the GradF-UFG with a stepsize $\tau = 1$.
\end{proposition}
\begin{proof}
The proof is similar to that of Proposition \ref{prop_frame_gradient_flow} and hence omitted. 
\end{proof}

From Proposition \ref{prop_ee_conv_gradient_flow}, we see with specific choices of the weight matrices, one can recover the evolution of EE-UFG. We can then follow the similar analysis in Section \ref{sect_dirichlet_energy_framelet} to show EE-UFG can induce both LFD and HFD dynamics. Nevertheless, we highlight that it is not the perturbation that steers the dynamics from over-smoothing (as shown in Proposition \ref{prop_smoothing_perturb}), but the separate weight matrices for low-frequency and high-frequency components of the signal.

In fact, we can further compare the framelet generalized energy with and without perturbation. Let $\gE_{\rm frame}$ denote the energy $\gE^{\rm tot}$ with $\bOmega_{0,J} = \bOmega_{r,j} = \bI_c$ that leads to the framelet convolution (as in Proposition \ref{prop_frame_gradient_flow}), and let $\gE_{\rm frame}^\epsilon$ denotes its perturbed energy as in Proposition \ref{prop_ee_conv_gradient_flow}. Then we have 
\begin{equation}
    \gE^\epsilon_{\rm frame}(\bH) =  \gE_{\rm frame}(\bH) + \frac{\epsilon}{2} \sum_{i=1}^n \Big( \langle \gW_{0,J} \bh_i, \bW_{0,J} \gW_{0,J} \bh_i \rangle - \sum_{r,j} \langle \gW_{r,j} \bh_i, \bW_{r,j} \gW_{r,j} \bh_i \rangle  \Big). \label{eq_frame_perturb_energy_compare}
\end{equation}
When all the weight matrices are $\bW_{r,j} = \bI_c$ as in the case for Dirichlet energy, the perturbation indeed provides an enhancement in the energy. Nonetheless, for the general case in \eqref{eq_frame_perturb_energy_compare}, whether the perturbation increases or decreases the energy depends on the relative spectrum of the weight matrices for different framelets. 

We finally remark that, as opposed to setting a fixed $\epsilon$ as in \cite{chendirichlet}, the proposed energy allows more flexible choices of the weights that adapt to different datasets.

\subsection{The influence of the activation functions}
\label{activation_sect}

So far, we have only considered the linearized dynamics of gradient flow or the nonlinearity is decoupled from the dynamics via an encoder model (as in \eqref{eq_gradf_ufg}). Here we study an activated GradF-UFG model, where nonlinearity is added at each timestamp. Although the dynamics no longer qualifies as a gradient flow, we show the analysis is not affected. 
Particularly, we consider an activated gradient flow similar as in \cite{di2022graph}:
\begin{equation}
    \dot{\bH}(t) = \sigma \big( -\nabla \gE^{\rm tot}(\bH(t)) \big) = \sigma \Big( - \sum_{(r,j) \in \gI} \big( \gW_{r,j}^\top \gW_{r,j} \bH(t) \bOmega_{r,j} - \gW_{r,j}^\top \widehat{\bA}  \gW_{r,j} \bH(t) \bW_{r,j} \big)  \Big) \label{eq_activated_grad_flow}
\end{equation}

Then we show in the following theorem that for common activation functions, the conclusions from the paper remains the same. 
\begin{theorem}
\label{thm_act_grad_flow}
For any activation function $\sigma: \sR \xrightarrow{} \sR$ satisfying $x \sigma(x) \geq 0$ for all $x \in \sR$, then suppose $\bH(t)$ solves \eqref{eq_activated_grad_flow}, we have $\gE^{\rm tot}(\bH(t))$ is decreasing in $t$. For its discretization, $\bH(t+ \tau) = \bH(t) + \tau \sigma (- \nabla \gE^{\rm tot}(\bH(t)))$, we have $\gE^{\rm tot}(\bH(t)) \leq \gE^{\rm tot}(\bH(t)) + C_{m} \| \bH(t+\tau) - \bH(t) \|^2$, where $C_m$ is the largest eigenvalue magnitude of $\frac{1}{2} \sum_{(r,j) \in \gI} \big( \bOmega_{r,j} \otimes \gW_{r,j}^\top \gW_{r,j} - \bW_{r,j} \otimes \gW_{r,j}^\top \widehat{\bA} \gW_{r,j} \big)$.
\end{theorem}
\begin{proof}
The proof follows the same strategy as in \cite[Proposition 3.3]{di2022graph}. For the continuous case, let $\bz(t) = - \vec( \nabla \gE^{\rm tot}( \bH(t) ) )$. Then 
\begin{equation*}
    \frac{d}{dt} \vec(  \gE^{\rm tot} (\bH(t)) ) = \langle \vec( \nabla \gE^{\rm tot} (\bH(t)) ), \vec( \dot{\bH}(t) )  \rangle = - \langle \bz(t), \sigma(\bz(t)) \rangle \leq 0.
\end{equation*}
For the discrete case $\bH(t+ \tau) = \bH(t) + \tau \sigma (- \nabla \gE^{\rm tot}(\bH(t)))$, we have 
\begin{align*}
    &\gE^{\rm tot}(\bH(t+\tau)) \\
    &= \left\langle \vec(\bH(t+\tau)) , \frac{1}{2} \sum_{(r,j) \in \gI} \big( \bOmega_{r,j} \otimes \gW_{r,j}^\top \gW_{r,j} - \bW_{r,j} \otimes \gW_{r,j}^\top \widehat{\bA} \gW_{r,j} \big) \vec(\bH(t+ \tau)) \right\rangle \\
    &= \gE^{\rm tot} (\bH(t)) + \tau^2 \langle \sigma(\bz(t)), \bS  \sigma(\bz(t)) \rangle + 2 \tau \langle \vec(\bH(t)), \bS \sigma(\bz(t)) \rangle \\
    &= \gE^{\rm tot} (\bH(t)) + \left\langle \vec\big(\bH(t+\tau) \big)- \vec\big(\bH(t) \big), \bS  \big( \vec(\bH(t+\tau)) - \vec(\bH(t)) \big) \right\rangle - 2 \tau \langle \sigma(\bz(t)), \bz(t) \rangle \\
    &\leq \gE^{\rm tot} (\bH(t)) + \| \bS \| \|  \vec(\bH(t+\tau)) - \vec(\bH(t)) \|^2 
\end{align*}
where we let $\bS =\frac{1}{2} \sum_{(r,j) \in \gI} \big( \bOmega_{r,j} \otimes \gW_{r,j}^\top \gW_{r,j} - \bW_{r,j} \otimes \gW_{r,j}^\top \widehat{\bA} \gW_{r,j} \big)$, which is symmetric. 
\end{proof}

The above results suggest that for common activation functions satisfying $x \sigma(x) \geq 0$, such as ReLU and tanh, the activated gradient flow still minimizes the energy.

\subsection{A even more general framelet energy with source term}

We may consider a even more general energy via graph framelets by including a source term $\bH(0)$. The source term has been shown to avoid over-smoothing and increase the expressive power \cite{chen2020simple,gasteiger2018predict,thorpe2021grand}. Particularly, we define 
\begin{equation*}
    \gE_{r,j} (\bH) = \frac{1}{2} \trace\big( (\gW_{r,j} \bH)^\top \gW_{r,j} \bH \bOmega_{r,j} \big) - \frac{1}{2} \trace \big( (\gW_{r,j} \bH)^\top \widehat{\bA} \gW_{r,j} \bH \bW_{r,j} \big) + \beta \trace( (\gW_{r,j} \bH )^\top \bH(0) \widetilde{\bW}_{r,j} )
\end{equation*}
for all $(r,j) \in \gI$. The corresponding GradF-UFG model is given by its discretized gradient flow as 
\begin{equation*}
   \text{GradF-UFG} \, : \,  \bH(t+\tau) = \bH(t) - \tau \sum_{(r,j) \in \gI} \Big( \gW_{r,j}^\top \gW_{r,j} \bH(t) \bOmega_{r,j} - \gW_{r,j}^\top \widehat{\bA}  \gW_{r,j} \bH(t) \bW_{r,j} - \beta  \gW_{r,j}^\top \bH(0) \widetilde{\bW}_{r,j} \Big)
\end{equation*}
where $\bH(0) = \texttt{MLP}_\theta(\bX)$.  
% \GaoC{Could we leave this for another small project? We discuss what conditions to put on in training $W$s?}
% \AHcomment{The source term along would not be a strong contribution in my view. Plus, adding the source term makes this work a complete generalization to \cite{chendirichlet}, which also considers the source term. For the second one, yes, it would be interesting to have a better control over the $W$s.}

\section{A gradient flow perspective on the spectral framelet convolution}
\label{sect_spec_grad_flow}

We have thus far focused on spatial framelet convolution \cite{chendirichlet}. For the spectral framelet convolution \cite{zheng2021framelets}, one can similarly characterize its behaviors as an energy gradient flow. Recall the (linearized) spectral framelet convolution (with fixed weight) is given by 
\begin{equation}
    \bH(\ell + 1) = \gW_{0,J}^\top {\rm diag}(\btheta_{0,J}) \gW_{0,J} \bH(\ell) \bW +  \sum_{r,j} \gW_{r,j}^\top {\rm diag}(\btheta_{r,j}) \gW_{r,j} \bH(\ell) \bW, \label{eq_spec_frame}
\end{equation}
where $\btheta_{r,j} \in \sR^n$ is the filter coefficients for framelets at $(r,j)$. In particular, the energy that leads to the evolution in \eqref{eq_spec_frame} is given by  
\begin{equation}
    \widetilde{\gE}^{\rm tot}(\bH) =  \frac{1}{2} \sum_{(r,j) \in \gI} \trace \left(  (\gW_{r,j} \bH)^\top  \gW_{r,j} \bH  -   (\gW_{r,j} \bH)^\top {\rm diag}(\btheta_{r,j}) \gW_{r,j} \bH \bW  \right). \label{eq_tot_spec_energy}
\end{equation}
% One can verify that the discretization of gradient flow of $\widetilde{\gE}^{\rm tot}(\bH)$ with a stepsize $\tau = 1$ leads to the spectral framelet convolution in \eqref{eq_spec_frame}. 

We next show, similar to spatial framelet convolution, the spectral convolution in \eqref{eq_spec_frame} can lead to both LFD and HFD dynamics. 

\begin{theorem}
\label{thm_sp_frame_lfdhfd}
The spectral graph framelet convolution \eqref{eq_spec_frame} with Haar-type filter can induce both LFD and HFD dynamics. Specifically, let $\btheta_{0,J} = \bones_n$ and $\btheta_{r,j} = \theta \bones_n$ for $r=1,...,L, J = 1,...,J$ where $\bones_n$ is a vector of all $1$s. Suppose $\theta \geq 0$. Then when $\theta \in [0,1)$, the spectral framelet convolution is LFD and when $\theta > 1$, the spectral framelet convolution is HFD.
\end{theorem}
\begin{proof}
We here focus on the Haar filter with $J = 1$ scale. For $J = 2$, the proof is exactly the same.
We can write the discretized gradient flow of \eqref{eq_tot_spec_energy} with a stepsize $\tau$ as
\begin{align*}
    \vec \big( \bH(m\tau) \big) &= \tau^m \Big( \bW \otimes (\gW_{0,1}^\top \gW_{0,1} + \theta \gW_{1,1}^\top \gW_{1,1}) \Big)^m \vec \big( \bH(0) \big) \\
    &= \tau^m \sum_{k,i} \Big(\lambda_k^W  \big(  \cos^2(\lambda_i/8) + \theta \sin^2(\lambda_i/8) \big) \Big)^m c_{k,i}(0) \bphi_k^W \otimes \bu_i.
\end{align*}
where we let $\{ (\lambda_k^W, \bphi_k^W) \}_{k = 1}^c$ as the eigen-pairs of $\bW$ and denote $c_{k,i}(0) = \langle  \vec(\bH(0)), \bphi_k^W \otimes \bu_i \rangle$. Now we see
\begin{align*}
    |\lambda_k^W (\cos^2(\lambda_i/8) + \theta \sin^2(\lambda_i/8) )| \leq \Delta^W (\cos^2(\lambda_i/8) + \theta \sin^2(\lambda_i/8) ).
\end{align*}
where we let $\Delta^W := \max_k |\lambda_k^W|$. It is easy to see when $\theta > 1$, $\cos^2(\lambda_i/8) + \theta \sin^2(\lambda_i/8)$ is (monotonically) 
% \EScomment{better to add a monotonic term here? Since otherwise the max value may not belongs to $\rho$?}
increasing in $\lambda_i \in [0, \rho_L]$ and its maximum is achieved at $\lambda_i = \rho_L$. On the other hand, when $\theta \in [0,1)$, the function is (monotonically) decreasing and the maximum is achieved at $\lambda_i = 0$. The proof is complete by following the same steps as for Theorem \ref{thm_frame_lfdhfd}. 
\end{proof}

\begin{remark}
From Theorem \ref{thm_sp_frame_lfdhfd}, we can draw a similar conclusion as for the spatial framelet convolution. That is, when the weights on the high-pass framelets are higher than that of the low-pass, the spectral framelet convolution is HFD. Otherwise, the model is LFD. This shows an equivalence between the spectral and spatial framelet convolution from the perspective of gradient flow. 
\end{remark}

\section{Concluding remarks}
In this paper, we have theoretically analyzed the behaviors of framelet-based GNNs via the framework of gradient flows and energy minimization. We have shown the framelet convolutions, either spatial or spectral, can accommodate both low-frequency-dominated (LFD) and high-frequency-dominated (HFD) dynamics, unlike many existing works that can only be LFD. This corroborates the good empirical performance of framelet-based models on heterophilic graphs and on mitigating the over-smoothing issue. The main underlying reason is that the framelet decomposition allows separate weight matrices on different frequencies. More specifically, when the weights on high-frequency components are relatively higher than the low-pass components, the dynamics can be HFD. In addition, we also propose a framelet generalized energy where its gradient flow includes many existing models as special cases. We have explained from various aspects why the proposed model generally leads to more flexible dynamics. 

It should be noted that for simplicity of analysis, we have only focused on the Haar-based filter and when the weight matrices share the same eigen-space. This suffices for the scope of this work. It would be interesting to explore necessary conditions (if available) for framelet-based models to be LFD/HFD. We hope this work would motivate more theoretical understanding on the capability of multi-scale spectral GNNs via alternative frameworks.

\bibliographystyle{plain}
%\bibliography{references}

\end{document}